\documentclass[]{config/antgroup}
\usepackage{config/antgroup}
\newcommand*\justify{%
  \fontdimen2\font=0.4em% interword space
  \fontdimen3\font=0.2em% interword stretch
  \fontdimen4\font=0.1em% interword shrink
  \fontdimen7\font=0.1em% extra space
  \hyphenchar\font=`\-% allowing hyphenation
}

\renewcommand{\texttt}[1]{%
\begingroup
\ttfamily
\begingroup\lccode`~=`/\lowercase{\endgroup\def~}{/\discretionary{}{}{}}%
\begingroup\lccode`~=`[\lowercase{\endgroup\def~}{[\discretionary{}{}{}}%
\begingroup\lccode`~=`.\lowercase{\endgroup\def~}{.\discretionary{}{}{}}%
\catcode`/=\active\catcode`[=\active\catcode`.=\active
\justify\scantokens{#1\noexpand}%
\endgroup
}

\graphicspath{{doc/}}

\PassOptionsToPackage{numbers, compress}{natbib}

\usepackage[utf8]{inputenc}         %
\usepackage[T1]{fontenc}            %
\usepackage{graphicx}
\usepackage{todonotes}
\usepackage{multirow}
\usepackage{hyperref}
\usepackage[capitalize,noabbrev,nameinlink,sort]{cleveref}
\usepackage{subcaption}
\usepackage{multicol}
\usepackage{array}
\usepackage{enumitem}
\usepackage{algorithm}
\usepackage{algorithmic}
\usepackage{tabularx}
\usepackage{listings}
\usepackage{makecell}
\usepackage{xspace}
\usepackage{colortbl}
\usepackage{fontawesome5}
\usepackage{pifont}
\usepackage[normalem]{ulem}
\useunder{\uline}{\ul}{}
\usepackage{tablefootnote}
\usepackage{tcolorbox}
\usepackage{arydshln}
\usepackage[toc, page, header]{appendix}
\usepackage{tabularx}
\usepackage{listings}

\usepackage{amsmath}
\usepackage{amssymb}
\usepackage{amsfonts}                               % blackboard math symbols
\usepackage{amsthm}
\usepackage[mathcal]{eucal}
\usepackage{mathrsfs}
\usepackage{bm}                                     % bm command
\usepackage{blkarray}                               % to support matrix
\usepackage{nicefrac}                               % compact symbols for 1/2, etc.
\usepackage{bbm}

\newtheorem{theorem}{Theorem}[section]

\newtheorem{proposition}[theorem]{Proposition}

\usepackage{wrapfig}
\usepackage{graphicx}                               % include pdf figures
\usepackage{caption}
\usepackage{tikz}
\usepackage{circuitikz}
\usetikzlibrary{patterns,snakes}
\usetikzlibrary{positioning,calc,fit,decorations.pathmorphing,shapes.geometric, shapes.gates.logic.US, calc}
\usetikzlibrary{arrows,arrows.meta,decorations.markings,shapes,shapes.arrows}
\usetikzlibrary{decorations,decorations.pathreplacing}
\usetikzlibrary{backgrounds}
\usepackage{filecontents}                           % support to pgfplots
\usepackage{pgfplots}
\usepackage{pgfplotstable}
\usepgfplotslibrary{groupplots}
\usepackage{scalefnt}
\pgfplotsset{compat=newest}
\usepackage{xcolor}

\definecolor{firstcolor}{HTML}{C3423F}
\definecolor{secondcolor}{HTML}{2A4B8C}
\definecolor{aworld_blue}{HTML}{4e81ff}
\definecolor{aworld_cyan}{HTML}{41d7fa}
\definecolor{aworld_teal}{HTML}{5fede4}
\definecolor{coral}{RGB}{255,127,80}
\definecolor{darkgreen}{RGB}{0,100,0}
\definecolor{darkyellow}{RGB}{204,153,0}
\definecolor{salmon}{RGB}{250,128,114}
\definecolor{darkred}{RGB}{150,0,0}

\hypersetup{
  colorlinks=true,
  linkbordercolor=aworld_blue,
  citebordercolor=aworld_cyan,
  urlbordercolor=aworld_teal
}

\newcommand{\ours}{\textsc{CrEST}\xspace}

\definecolor{improvementblue}{RGB}{55,126,184}    % Blue (#377eb8)
\definecolor{degradationorange}{RGB}{230,85,13}   % Orange (#e6550d)
\newcommand{\improvement}[1]{\textcolor{improvementblue}{\textbf{#1}}}

\def\eqref#1{equation~\ref{#1}}
\def\1{\bm{1}}

\DeclareMathAlphabet{\mathsfit}{\encodingdefault}{\sfdefault}{m}{sl}
\SetMathAlphabet{\mathsfit}{bold}{\encodingdefault}{\sfdefault}{bx}{n}

\begin{document}
\title{Teach the Magnitude, Not the Direction: Verifier-Anchored Credit Assignment for Multi-Turn Multi-Step LLM Agents}

% TODO: 请在此处填写作者信息（当前为占位，参考模板）
\author{
  Zechuan Wang$^{1,3,\ast}$, Siyuan Lu$^{1,2,3,4,\ast}$, Hongxuan Zhang$^{3,5}$,
  Linjian Mo$^{3}$, Chenyi Zhuang$^{3}$, Leilei Gan$^{1,\dagger}$
  % 添加更多作者...
}

% TODO: 根据需要配置脚注（当前为占位）
{\renewcommand\thefootnote{}
  \footnotetext{$^\ast$Equal contributions. Work was done during  Zechuan and Siyuan's internship in Ant Group.
  }
  \footnotetext{$^\dagger$Corresponding author: \texttt{leileigan@zju.edu.cn}.}
}

% TODO: 请在此处填写机构信息（当前为占位，参考模板）
\affiliation{$^1$Zhejiang University\quad}
\affiliation{$^2$Shanghai Innovation Institute\quad}
\affiliation{$^3$\raisebox{-0.2em}{\includegraphics[height=1.1em]{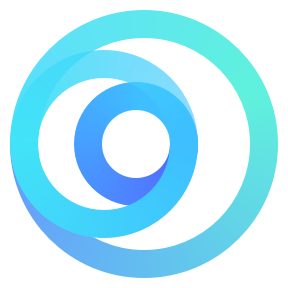}}AWorld Team, Inclusion AI\\}
\affiliation{$^4$Westlake University\quad}
\affiliation{$^5$Nanjing University}

% 添加更多机构...

\maketitle

% TODO: 可选 - 添加您的项目GitHub链接或其他链接
% \begin{center}
%   \vspace{-1.5em}
%   \href{https://github.com/}{\faGithub\ \texttt{your-project}}
%   \vspace{0.5em}
% \end{center}

%\clearpage
%\tableofcontents
%\newpage
%%%%%%%%%%%%%%%%%%%%%%%%%%%%%%%%%%%%%%%%%%%%%%%%%%%%%%%%%%%%%%%%%%
%%%%%%%%%%%%%%%%%%%%%%%%%%%%%%%%%%%%%%%%%%%%%%%%%%%%%%%%%%%%%%%%%%
%%%%%%%%%%%%%%%%%%%%%%%%%%%%%%%%%%%%%%%%%%%%%%%%%%%%%%%%%%%%%%%%%%
% main document

% 从第二页开始使用带页眉的样式
\begin{abstract}
  Reinforcement learning with verifiable rewards (RLVR) directly optimizes verifier-defined outcomes for multi-turn tool-use agents, yet its trajectory-level credit assignment conflates heterogeneous per-turn outcomes into a single reward signal. On-policy distillation provides dense per-token supervision but uses a teacher-defined token-level objective and can be prone to gradient concentration collapse. We introduce \ours, a hierarchical credit assignment framework that anchors each sampled token's scalar advantage sign to verified rewards while incorporating dense magnitude signals from a privileged self-teacher. \ours resolves credit at two levels: turn-segmented verified advantages address inter-turn dilution, while entropy-gated self-teacher modulation refines intra-turn token contributions. Experiments on BFCL V3 and WildToolBench show that \ours consistently outperforms both RL and distillation baselines across two model scales, with the largest gains on long-trajectory and strict session-level metrics. Our work demonstrates that the teacher's role in policy optimization can be reduced from determining token-level reinforcement or suppression to modulating update magnitudes, enabling dense credit assignment while keeping verifier rewards as the training objective.
\end{abstract}

\section{Introduction}

Large language models~\citep{hurst2024gpt, comanici2025gemini, qwen2025qwen25, liu2024deepseek} are evolving from pure question-answering systems into autonomous agents capable of reasoning, acting, and interacting with diverse environments~\citep{merrill2026terminal, jimenez2024swe, xie2024osworld, liu2026gem, froger2025scaling, towers2026gymnasium}.
These scenarios share a common structure: agents must handle \textbf{multi-turn} sessions (sequences of user requests) where each turn requires \textbf{multi-step} execution (chains of actions and observations)~\citep{shridhar2021alfworld,wei2025browsecomp, xie2024osworld,openclaw2026}.

\begin{figure*}[!t]
    \centering
    \includegraphics[width=0.95\textwidth]{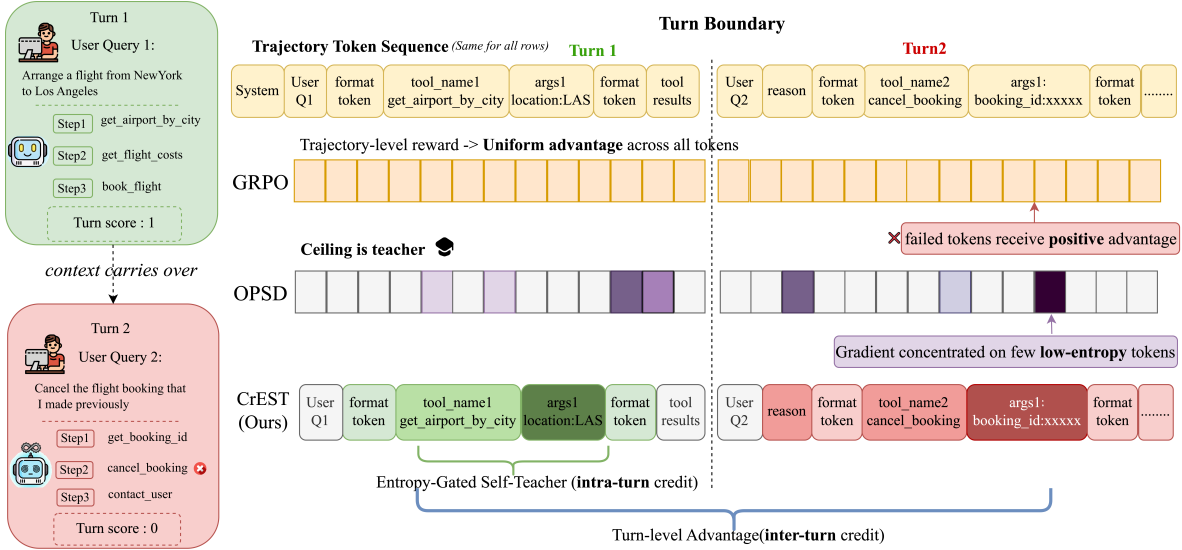}
    \caption{Comparison of credit assignment strategies on a 2-turn trajectory (Turn 1 succeeds, Turn 2 fails). \textbf{GRPO} broadcasts a single trajectory-level reward uniformly, erroneously reinforcing failed-turn tokens. \textbf{OPSD} provides a dense but teacher-defined token-level signal and concentrates gradient on low-entropy format tokens. \textbf{\ours{} (Ours)} combines turn-level advantages (inter-turn credit) with entropy-gated self-teacher modulation (intra-turn credit), correctly assigning negative advantage to Turn 2 while focusing gradient mass on high-uncertainty content tokens in Turn 1.}
    \label{fig:overview}
\end{figure*}

Reinforcement learning~\citep{schulman2017proximal,ahmadian2024back}, particularly RL with verifiable rewards (RLVR)~\citep{shao2024deepseekmath,qian2026toolrl}, is now the dominant post-training paradigm for tool-use agents~\citep{liu2025understanding,zhang2025nemotron,yu2026dapo}, directly optimizing verifier-defined outcome rewards. Yet standard objectives broadcast a single trajectory-level reward across all tokens.
Within a single turn, where all actions serve one query, this is merely noisy~\citep{lu2026dont}. In \textbf{multi-turn} sessions~\citep{li2023api, patil2025berkeley, yu2026benchmarking}, however, turns carry independent outcomes yet share one reward signal, making credit assignment ill-posed: on WildToolBench, no model exceeds 15\% session accuracy, and performance degrades sharply with turn count~\citep{laban2025llms}.

On-policy distillation (OPD)~\citep{lu2025onpolicy, li2026rethinkingopd, song2026survey} provides a natural credit assignment mechanism, where the teacher inherently scores each token in the student's trajectory and distinguishes which tokens are preferable from the teacher's perspective.
However, this requires a same-family, tokenizer-matched teacher~\citep{fu2026revisiting}, which is a rare constraint satisfied in practice.
On-policy self-distillation (OPSD)~\citep{zhao2026self, hubotter2026reinforcement, shenfeld2026self, penaloza2026privileged} removes this dependency by using the student itself as self-teacher, conditioned on privileged information such as ground truth or successful demonstrations from rollout group peers.
% Yet OPSD is brittle---without per-token KL clipping, it collapses within ${\sim}100$ steps due to gradient concentration~\citep{zhao2026self, brown2026sftrlopd}. 
% Moreover, OPSD restricts exploration~\citep{kim2026does} and its performance ceiling remains \emph{teacher-bounded}---the student cannot surpass the teacher's privileged distribution.
Yet OPSD is brittle and restricts exploration.
First, \citet{zhao2026self, brown2026sftrlopd} found that without per-token KL clipping, OPSD collapses within ${\sim}$100 training steps due to gradient concentration on a few tokens.
Second, OPSD optimizes a teacher-defined token-level target~\citep{kim2026does}, making its training behavior dependent on the privileged teacher distribution.

This raises a central question: \emph{Can we design a training approach that retains verified rewards as the optimization target while incorporating dense token-level signals to address hierarchical credit assignment in multi-turn multi-step LLM agent training?}

% Concretely, training GRPO on multi-turn trajectories faces a \textbf{two-level credit assignment problem}: 1)~\emph{inter-turn}, where trajectory-level reward averaging dilutes per-turn outcomes so that tokens in failed turns can receive positive advantage; and 2)~\emph{intra-turn}, where all tokens within a turn share the same advantage, leaving pivot decisions indistinguishable from low-entropy format tokens.

% We propose \ours{} (Hierarchical \textbf{Cr}edit assignment via \textbf{E}ntropy-gated \textbf{S}elf-\textbf{T}eacher), which addresses both levels through a unified policy gradient with structured per-token advantages: turn-segmented verified rewards resolve inter-turn dilution, while an entropy-gated self-teacher refines intra-turn token contributions.
% The two levels are complementary---neither alone recovers the full performance---and verified reward determines all gradient directions, ensuring the performance ceiling remains verifier-bounded without requiring an external same-family teacher.

We propose \ours{} (Hierarchical \textbf{Cr}edit Assignment via \textbf{E}ntropy-Gated \textbf{S}elf-\textbf{T}eacher), which addresses both levels through a unified policy gradient with structured per-token advantages (\cref{fig:overview}).
At the turn level, \ours{} uses turn-segmented verifier rewards to prevent rewards from being diluted across long multi-turn trajectories.
At the token level, to prevent gradient concentration on a few tokens, \ours{} uses an entropy-gated self-teacher to modulate token-level contributions within each turn.
Additionally, \ours{} uses verified rewards to determine the sign of each sampled token's scalar advantage while reducing the self-teacher's role to magnitude modulation, enabling dense credit assignment while keeping verifier rewards as the training objective and avoiding the need for an external same-family teacher.
The two levels of credit assignment are complementary---neither alone recovers the full performance.

Our main contributions are:

\begin{enumerate}[nosep, leftmargin=12pt]
    \item We identify and formalize the two-level credit assignment problem in multi-turn multi-step agentic RL, and provide a gradient-geometric analysis explaining why existing dense-signal methods fail in this setting (\cref{app:gradient_analysis}).
          \looseness=-1
    \item We propose \ours{}, a hierarchical credit assignment framework combining turn-segmented advantage with entropy-gated self-teacher token reweighting. \ours{} is verifier-anchored at the token level by construction and tunes only one teacher-influence hyperparameter ($\lambda$) beyond standard GRPO (\cref{sec:method}).
          \looseness=-1
    \item We report that \ours{} achieves 52.0\% average accuracy on BFCL V3 (Qwen3-4B) and 9.38\% session accuracy on WildToolBench (Qwen3-8B), obtaining the highest observed aggregate scores among the compared methods (\cref{sec:experiments}).
\end{enumerate}
\section{Related Work}
\label{sec:related_work}

\paragraph{Tool use for LLM-based agents.}
LLM tool-use research spans a complexity spectrum, from reasoning before function-calling~\citep{zhang2025nemotron, qian2026toolrl} to tool-integrated reasoning (TIR) for fulfilling user's query~\citep{li2025torl, jin2025search, feng2025retool, zheng2025deepresearcher}. Most relevant to our work, multi-turn multi-step benchmarks~\citep{zhang2025turnbench, yu2026benchmarking, patil2025berkeley} evaluate agents across sessions of multiple queries, each requiring its own execution chain.
These settings expose two challenges absent in single-turn work: cross-turn state consistency and error propagation across turns. Empirically, agent performance degrades substantially with turn count; \citet{yu2026benchmarking} report session accuracy below 15\% across 57 LLMs, motivating training-time methods that respect the turn structure of multi-turn sessions.

\paragraph{Agentic Reinforcement Learning.}

RL-based post-training has become the standard approach for aligning LLM agents with task objectives~\citep{shao2024deepseekmath,yu2026dapo,wang2025ragen}.
Some methods focus on reward design: \citet{lu2026dont} decompose trajectory rewards into fine-grained per-step process scores combining state correctness and execution accuracy, while \citet{qian2026toolrl} and \citet{zhang2025nemotron} ground rewards directly in environment feedback from tool execution.
Despite these advances, most existing methods compute a single advantage per trajectory (or per-step within a single turn), leaving the inter-turn credit assignment problem unaddressed---when turns within the same session have heterogeneous outcomes, trajectory-level or step-level averaging still mislabels tokens in failed turns.
MT-GRPO~\citep{wei2025reinforcing} attempts per-agent-step advantage computation, yet still lacks intra-step token-level credit assignment, treating all tokens within each step identically.

\paragraph{On-policy distillation.}
On-policy distillation (OPD)~\citep{lu2025onpolicy} replaces RL's sparse reward with dense per-token reverse KL against a same-family teacher, providing token-level credit assignment throughout the trajectory~\citep{li2026rethinkingopd}.
OPD converges faster than RL at moderate compute budgets, but it optimizes a teacher-defined token-level target and requires tokenizer-matched, recipe-matched teachers~\citep{brown2026sftrlopd}. When such teachers are unavailable, self-distillation methods use the student itself conditioned on privileged information like expert demonstrations~\citep{shenfeld2026sdft}, feedback~\citep{hubotter2026reinforcement} and ground-truth answers~\citep{zhao2026self}.
More recently, hybrid methods attempt to combine verifier-derived scalar advantage signs with distillation's dense signal: SDAR~\citep{lu2026self} gates self-distillation as an auxiliary loss alongside RL, and RLSD~\citep{yang2026rlsd} uses self-distillation for update magnitudes while relying on verifiable rewards for token-level reinforcement or suppression.
However, these hybrids either lack principled control over gradient concentration~\citep{zhao2026self,brown2026sftrlopd} or do not address the inter-turn credit assignment structure unique to multi-turn settings.

\section{Method}
\label{sec:method}

Multi-turn multi-step training faces a \emph{hierarchical} credit assignment problem: at the coarse level, trajectories mix successes and failures across turns; at the fine level, individual turns conflate high-value decisions with low-entropy formatting.
\ours{} addresses both levels through a unified policy gradient framework with structured per-token advantages (\cref{eq:objective}):
\begin{itemize}[nosep,leftmargin=12pt]
    \item \textbf{Inter-turn (coarse)}: Turn-segmented verified-reward advantages (\cref{sec:method:turn}) isolate per-turn outcomes, preventing dilution from trajectory-level averaging.
    \item \textbf{Intra-turn (fine)}: Entropy-gated self-teacher signals (\cref{sec:method:token}) refine per-token contributions, concentrating gradient mass on high-uncertainty content tokens.
\end{itemize}

The two levels are \emph{complementary}: without inter-turn segmentation, no amount of intra-turn refinement can rescue tokens in failed turns; without intra-turn reweighting, pivot decisions in successful turns remain indistinguishable from formatting tokens. Critically, verified reward determines the sign of each sampled token's scalar advantage while the self-teacher only modulates its magnitude, yielding verifier-anchored token reweighting (\cref{app:gradient_analysis}).

\begin{figure*}[t]
    \centering
    \includegraphics[width=0.95\textwidth]{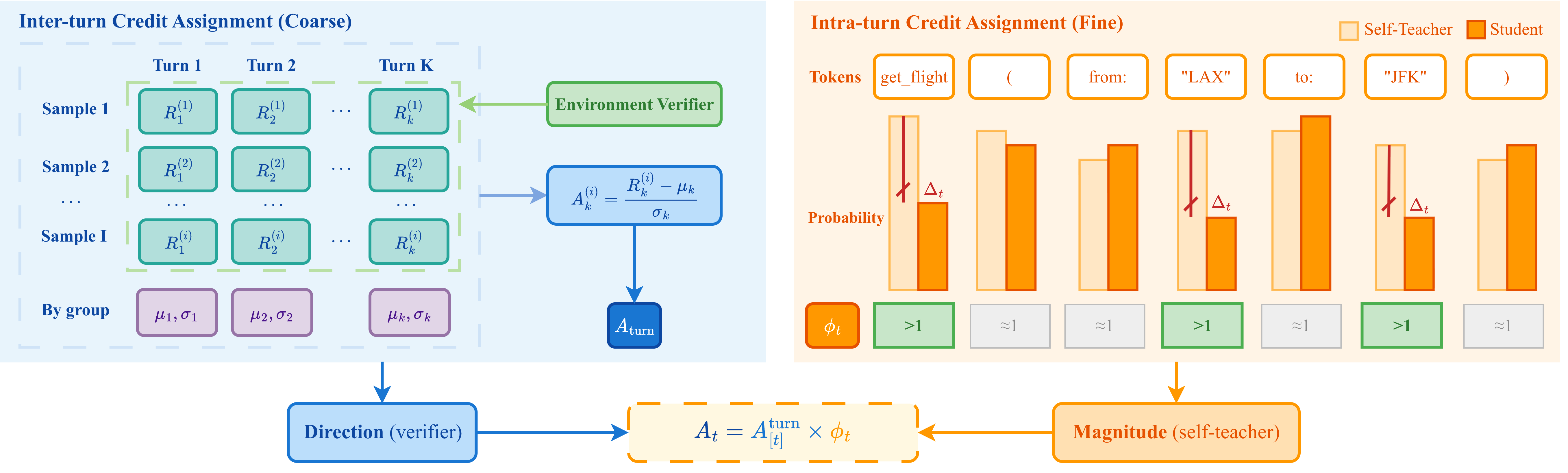}
    \caption{Overview of \ours. \textbf{Left (Inter-turn)}: A group of $G$ rollouts is scored by the environment verifier per turn; group-relative advantages $A_{\text{turn}}$ are computed independently for each turn, providing the verifier-derived scalar sign. \textbf{Right (Intra-turn)}: Within each turn, a privileged self-teacher (conditioned on ground truth) provides per-token divergence $\Delta_t$; an entropy gate concentrates modulation on high-uncertainty content tokens (e.g., \texttt{get\_flight}, \texttt{"JFK"}) while suppressing format tokens, yielding the \emph{magnitude} factor $\phi_t$. \textbf{Bottom}: The final per-token advantage $\mathcal{A}_t = A_{[t]}^{\text{turn}} \times \phi_t$ combines a verifier-derived scalar sign with teacher-modulated magnitude.}
    \label{fig:method}
\end{figure*}

\subsection{Unified Objective}
\label{sec:method:objective}

\ours{} is a standard on-policy policy gradient with a \emph{structured per-token advantage}. Given a prompt $x$ and an on-policy rollout $y \sim \pi_\theta(\cdot \mid x)$, the objective is:
\begin{equation}
    J(\theta) = \mathbb{E}_{\substack{x \sim \mathcal{D},\; y \sim \pi_\theta(\cdot|x)}} \left[\sum_t \mathcal{A}_t \cdot \log \pi_\theta(y_t \mid y_{<t}, x)\right]
    \label{eq:objective}
\end{equation}
where the per-token advantage $\mathcal{A}_t$ exposes the hierarchical structure:
\begin{equation}
    \mathcal{A}_t = \underbrace{A_{[t]}^{\mathrm{turn}}}_{\substack{\text{\textbf{inter-turn}:} \\ \text{\emph{which turn?}}}} \cdot\ \underbrace{\phi_t}_{\substack{\text{\textbf{intra-turn}:} \\ \text{\emph{which token?}}}}
    \label{eq:advantage_decomposition}
\end{equation}
$A_{[t]}^{\mathrm{turn}}$ is a turn-segmented verified-reward advantage (\cref{sec:method:turn}) that resolves \emph{which turn} contributed to success or failure. $\phi_t \in [1,\; 1{+}\lambda\epsilon]$ is a magnitude modulation factor (\cref{sec:method:token}) that refines \emph{which tokens within that turn} deserve stronger updates, derived from a privileged self-teacher gated by student uncertainty.

\begin{wraptable}{r}{0.5\columnwidth}
    \vspace{-0.8em}
    \centering
    \footnotesize
    \caption{Design space of hierarchical credit assignment. Methods differ in how they allocate credit across turns ($A_{[t]}^{\mathrm{turn}}$) and tokens ($\phi_t$). MT-GRPO computes per-turn advantages but propagates discounted future-outcome credit to earlier turns (\emph{cumulative}), whereas \ours{} computes each turn's advantage from that turn's reward only (\emph{independent}). OPD and OPSD share the same slot ($\phi_t \propto \Delta_t$ with no verified-reward base, hence a teacher-defined token-level target); they differ only in teacher source (external same-family vs.\ self with privileged context).}
    \label{tab:special_cases}
    \resizebox{\linewidth}{!}{
        \begin{tabular}{@{}lcc@{}}
            \toprule
            \textbf{Method}         & $A_{[t]}^{\mathrm{turn}}$ & $\phi_t$                                \\
            \midrule
            GRPO                    & trajectory-level          & $1$                                     \\
            MT-GRPO                 & per-turn (cumulative)     & $1$                                     \\
            OPD~/~OPSD              & ---                       & $\propto \Delta_t$                      \\
            \textbf{\ours{} (Ours)} & per-turn (independent)    & $1 + \lambda_t^{\mathrm{eff}}(w_t - 1)$ \\
            \bottomrule
        \end{tabular}
    }
    \vspace{-1.2em}
\end{wraptable}

The factorization in \cref{eq:advantage_decomposition} reveals why both levels are necessary: $A_{[t]}^{\mathrm{turn}}$ provides the correct \emph{sign} (reinforce or suppress) for each turn's tokens, but treats all tokens within a turn uniformly; $\phi_t$ refines the \emph{magnitude} of each token's contribution, but cannot change sign (proof in \cref{app:gradient_analysis}). Together, they anchor each sampled token's scalar advantage sign to verified rewards while allowing token-level magnitudes to benefit from dense teacher signals. \Cref{tab:special_cases} shows how existing methods occupy restricted subspaces of this design.

\subsection{Turn-Level Credit Assignment}
\label{sec:method:turn}

Instead of computing a single trajectory-level reward averaged across turns, we compute group-relative advantages \emph{within each turn independently} (\cref{fig:method}, left). For a trajectory containing turns $\{1, \dots, K\}$, each turn $k$ receives its own verified reward $R_k$ from the environment. Within a GRPO group of $G$ rollouts, the advantage for turn $k$ in rollout $i$ is:
\begin{equation}
    A_k^{(i)} = \frac{R_k^{(i)} - \mathrm{mean}(\{R_k^{(j)}\}_{j=1}^G)}{\mathrm{std}(\{R_k^{(j)}\}_{j=1}^G) + \epsilon}
    \label{eq:turn_advantage}
\end{equation}
All tokens within turn $k$ of rollout $i$ share the advantage $A_k^{(i)}$. This ensures that a failed turn receives negative advantage regardless of whether other turns in the same trajectory succeeded.

\subsection{Token-Level Credit: Self-Teacher Signal}
\label{sec:method:token}

The modulation factor $\phi_t$ uses a privileged self-teacher (the same model conditioned on ground-truth context) to refine per-token gradient magnitude within each turn (\cref{fig:method}, right):
\begin{equation}
    \phi_t = 1 + \lambda_t^{\mathrm{eff}} \cdot (w_t - 1)
    \label{eq:phi}
\end{equation}
When $\lambda_t^{\mathrm{eff}} = 0$, $\phi_t = 1$ (no teacher influence); when $\lambda_t^{\mathrm{eff}} > 0$ and $w_t > 1$, $\phi_t > 1$ (teacher amplifies the gradient). We now define the two components: the raw teacher signal ($\Delta_t \to w_t$) and the gating mechanisms that control it (\cref{sec:method:control}).

\paragraph{Teacher--student divergence.}
For each response token $t$, let $\pi_T(\cdot \mid h_t^T)$ denote the self-teacher---the same model conditioned on ground truth as privileged context, following~\citet{zhao2026self}---and $\pi_\theta(\cdot \mid h_t)$ the student:
\begin{equation}
    \Delta_t = \frac{\log \pi_T(y_t \mid h_t^T) - \log \pi_\theta(y_t \mid h_t)}{\tau}
    \label{eq:delta}
\end{equation}
where $\tau$ is a temperature that smooths the divergence signal.

\paragraph{Token weight.}
The raw teacher preference is converted into a bounded multiplicative weight:
\begin{equation}
    w_t = \mathrm{clip}\!\Big(\exp\!\big(\mathrm{sign}(A_{[t]}^{\mathrm{turn}}) \cdot \Delta_t\big),\; 1{-}\epsilon,\; 1{+}\epsilon\Big)
    \label{eq:weight}
\end{equation}
The $\mathrm{sign}(A_{[t]}^{\mathrm{turn}}) \cdot \Delta_t$ construction aligns teacher preference with the verifier-derived scalar sign: for positive-advantage tokens, teacher agreement increases $w_t$; for negative-advantage tokens, teacher agreement with the ``this was bad'' signal also increases $w_t$. This guarantees \textbf{token-wise sign preservation} (P1): $\mathrm{sign}(\mathcal{A}_t) = \mathrm{sign}(A_{[t]}^{\mathrm{turn}})$ for all $t$ (proof in \cref{app:gradient_analysis}).

\subsection{Stabilization via Selective Gating}
\label{sec:method:control}

The self-teacher provides dense per-token signals, but unconstrained use leads to concentration collapse~\citep{brown2026sftrlopd}: gradient mass flows to a small set of tokens where teacher--student divergence is largest, often low-entropy formatting rather than high-value decisions. We stabilize the intra-turn reweighting through a composed gating function $\lambda_t^{\mathrm{eff}}$ (\cref{eq:lambda_eff}) that controls \emph{where} (agreement with verifier) and \emph{how much} (student uncertainty) the teacher signal is applied.

\paragraph{Direction gate.}
\begin{equation}
    g_t^{\mathrm{dir}} = \mathbf{1}\!\left[\mathrm{sign}(A_{[t]}^{\mathrm{turn}}) \cdot \Delta_t > 0\right]
    \label{eq:direction_gate}
\end{equation}
When teacher and verifier disagree on the token-level reinforcement or suppression sign, $g_t^{\mathrm{dir}} = 0$ and the teacher signal is fully suppressed. This prevents the teacher signal from flipping the verifier-derived scalar advantage sign of any sampled token.

\paragraph{Entropy gate.}
To prevent gradient concentration on low-entropy format tokens, we modulate the reweighting strength by the student's token-level uncertainty, using surprisal $u_t = -\log \pi_\theta(y_t \mid h_t)$ as a lightweight proxy:
\begin{align}
    g_t^{\mathrm{ent}} & = \sigma\!\left(\frac{u_t - \mathbb{E}[u]}{\mathrm{std}(u) + \epsilon}\right) \label{eq:entropy_gate} \\
    m_t^{\mathrm{ent}} & = 1 + \rho\,(2\,g_t^{\mathrm{ent}} - 1) \label{eq:entropy_mod}
\end{align}
where $\sigma$ is the sigmoid function and statistics are computed over response tokens. The Z-score normalization centers surprisal around the mean and scales by standard deviation, mapping to $[0,1]$ via sigmoid. For low-uncertainty format tokens, $m_t^{\mathrm{ent}} < 1$ reduces modulation strength; for high-uncertainty content tokens, $m_t^{\mathrm{ent}} > 1$ increases it. The range $[1{-}\rho,\; 1{+}\rho]$ adjusts the strength of selective amplification without suppressing any token's base verifier-derived contribution.

\paragraph{Composed gating.}
The direction and entropy gates compose into a single effective gating function:
\begin{equation}
    \lambda_t^{\mathrm{eff}} = \mathrm{clip}\!\left(\lambda \cdot \underbrace{g_t^{\mathrm{dir}} \cdot m_t^{\mathrm{ent}}}_{\text{safety \& focus}},\; 0,\; \lambda\right)
    \label{eq:lambda_eff}
\end{equation}
This yields $\phi_t = 1 + \lambda_t^{\mathrm{eff}}(w_t - 1) \in [1, 1{+}\lambda\epsilon]$, ensuring \textbf{three formal properties} (\cref{app:gradient_analysis}): (P1)~direction preservation ($\mathrm{sign}(\mathcal{A}_t) = \mathrm{sign}(A_{[t]}^{\mathrm{turn}})$ always); (P2)~the absolute perturbation bound $\|\sum_t(\phi_t-1)g_t\| \leq \lambda\epsilon\sum_t\|g_t\|$, where $g_t=A_{[t]}^{\mathrm{turn}}\nabla_\theta\log\pi_\theta(y_t\mid h_t)$; and (P3)~sign-consistent amplification ($\phi_t \geq 1$ when active). With $\lambda{=}0.3$ and $\epsilon{=}0.28$, the $8.4\%$ value bounds only the change to each token's advantage coefficient, not the relative bias of the aggregate gradient.

$\lambda$ is the \textbf{only tunable hyperparameter} controlling teacher influence strength. The remaining values---$\tau{=}2.0$ (divergence temperature), $\epsilon{=}0.28$ (clip bound, inherited from OPSD stabilization), $\rho{=}0.5$ (entropy redistribution range)---are fixed constants requiring no tuning (ablation in \cref{sec:experiments}).

The complete training procedure, consolidating all equations into executable pseudocode, is provided in \cref{app:algorithm}.
\section{Experiments}
\label{sec:experiments}

\subsection{Experimental Setup}

\paragraph{Benchmarks.}
We conduct experiments on two multi-turn multi-step tool-use benchmarks: (1) \textbf{Berkeley Function-Calling Leaderboard (BFCL) V3}~\citep{patil2025berkeley}. Under our processed protocol, we use 100 fixed IDs from the multi-turn \textit{Base} split for training and 400 non-overlapping examples for evaluation: 100 examples each from \textit{Base}, \textit{Missing Functions}, \textit{Missing Parameters}, and \textit{Long-Context}. (2) \textbf{WildToolBench}~\citep{yu2026benchmarking} contains 256 multi-turn sessions designed to capture real-world interaction challenges, including compositional tasks, implicit intent through pronoun reference or ellipsis, and task switching between goal-directed queries and casual chat. The benchmark evaluates both action correctness and parameter accuracy. We allocate 128 sessions for training and reserve the remaining ones for evaluation.

The turn index is semantically aligned within each rollout group: all rollouts use the same fixed session and the same ordered list of user turns, so turn $k$ denotes the same benchmark request across the group. This assumption is specific to these benchmarks and does not directly extend to variable turn structures, terminal-only rewards, or strongly coupled turns without local verification. BFCL rewards compare post-execution environment states with annotated per-turn target executions, including result-coverage checks for read-only tools; WildToolBench rewards path-match calls against human-verified valid paths and score action correctness, parameter correctness, and path completion.

\paragraph{Base models.}
We evaluate on two model scales from the Qwen3 family~\citep{yang2025qwen3}: \textit{Qwen3-4B-Instruct},\footnote{\scriptsize\url{https://huggingface.co/Qwen/Qwen3-4B-Instruct-2507}} an instruction-tuned model, and \textit{Qwen3-8B}, a thinking model. Both have native tool-calling support and represent practical deployment sizes where post-training efficiency matters most.

\paragraph{Baselines.}
We compare against four training methods spanning RL-based and distillation-based paradigms:

\textbf{GRPO}~\citep{shao2024deepseekmath} applies standard Group Relative Policy Optimization with binary trajectory-level reward (1 if all turns succeed, 0 otherwise). Advantages are averaged across all tokens in the trajectory regardless of per-turn outcomes.

\textbf{MT-GRPO}~\citep{wei2025reinforcing} computes advantages at the per-agent-step granularity, providing finer-grained credit assignment than trajectory-level methods. However, within each step, all tokens still share the same advantage, lacking token-level differentiation.

\textbf{EnvTuning}~\citep{lu2026dont} uses fine-grained process rewards that combine per-turn state correctness and execution accuracy, aggregating per-turn scores into a trajectory-level reward via weighted averaging. Despite richer reward signals, it still computes a single trajectory-level advantage.

\textbf{OPD}~\citep{lu2025onpolicy} performs on-policy distillation with a same-family, same-variant teacher, providing dense per-token supervision via reverse KL. To satisfy the tokenizer-matched constraint, we pair each student with a teacher of the same model variant: Qwen3-4B-Instruct uses Qwen3-235B-A22B-Instruct\footnote{\scriptsize\url{https://huggingface.co/Qwen/Qwen3-235B-A22B-Instruct-2507}} as teacher, while Qwen3-8B uses Qwen3-32B. OPD therefore optimizes a teacher-defined token-level target.

\textbf{OPSD}~\citep{zhao2026self} uses on-policy self-distillation where the student itself, conditioned on ground-truth context, serves as the privileged teacher. This removes the external teacher dependency but risks training instability from gradient concentration.

\paragraph{Implementation details.}
All methods are trained with the same on-policy rollout infrastructure, group size $G{=}16$, and learning rate schedule. For \ours, we set $\lambda{=}0.3$, $\epsilon{=}0.28$, $\tau{=}2.0$, and $\rho{=}0.5$; only $\lambda$ is tuned, while the other values are fixed design choices across all experiments. Each method uses a single training seed. We evaluate each trained policy with three decodes at temperature $10^{-6}$ to measure decoding consistency. Full training configuration and OPSD stabilization details are provided in \cref{app:implementation}, with training cost and throughput reported in \cref{tab:training_cost}.

\begin{table*}[!t]
    \centering
    \small
    \caption{Main results on BFCL V3 multi-turn and WildToolBench benchmarks. Best results per model in \textbf{bold}, second-best \underline{underlined}. Methods are grouped by training paradigm: RL-based (top), distillation-based (middle), and our method (bottom).}
    \label{tab:bfcl_v3_results}
    \resizebox{\textwidth}{!}{
        \begin{tabular}{l ccccc cc}
            \toprule
            \multirow{2}{*}{\textbf{Method}} & \multicolumn{5}{c}{\textbf{BFCL V3 Multi-Turn (\%)}} & \multicolumn{2}{c}{\textbf{WildToolBench (\%)}}                                                                                                                                                                                                    \\
            \cmidrule(lr){2-6} \cmidrule(lr){7-8}
                                             & \textbf{Average}                                     & \textit{Base}                                   & \textit{Miss Func}                     & \textit{Miss Param}                 & \textit{Long Context}               & \textbf{Task Acc.}                    & \textbf{Session Acc.}               \\
            \midrule
            \textit{Qwen3-4B-Instruct}       & 22.12                                                & 26.5                                            & 21.0                                   & 15.5                                & 25.5                                & 37.89                                 & 3.13                                \\
            \rowcolor{blue!6}
            \, + GRPO                        & 43.63                                                & 53.5                                            & 42.5                                   & 31.5                                & 47.0                                & 40.23                                 & 4.69                                \\
            \rowcolor{blue!6}
            \, + MT-GRPO                     & \underline{49.25}                                    & \underline{63.0}                                & 46.0                                   & 35.0                                & \underline{53.0}                    & 43.36                                 & \underline{6.25}                    \\
            \rowcolor{blue!6}
            \, + EnvTuning                   & 47.25                                                & 60.0                                            & \underline{47.0}                       & 32.0                                & 50.0                                & \underline{44.92}                     & 4.69                                \\
            \rowcolor{orange!6}
            \, + OPD                         & 44.50                                                & 52.0                                            & 43.0                                   & \underline{38.0}                    & 45.0                                & 42.58                                 & \underline{6.25}                    \\
            \rowcolor{orange!6}
            \, + OPSD                        & 38.75                                                & 46.0                                            & 41.0                                   & 26.0                                & 42.0                                & 38.88                                 & 4.69                                \\
            \rowcolor{cyan!8}
            \, + \textbf{\ours}              & \textbf{52.00} \improvement{(+29.88)}                & \textbf{67.0} \improvement{(+40.5)}             & \textbf{48.0} \improvement{(+27.0)}    & \textbf{38.0} \improvement{(+22.5)} & \textbf{60.0} \improvement{(+34.5)} & \textbf{48.44} \improvement{(+10.55)} & \textbf{7.03} \improvement{(+3.90)} \\
            \midrule
            \textit{Qwen3-8B}                & 33.38                                                & 41.5                                            & 38.5                                   & 27.0                                & 26.5                                & 43.75                                 & 4.69                                \\
            \rowcolor{blue!6}
            \, + GRPO                        & 43.25                                                & 53.0                                            & 46.0                                   & 36.0                                & 38.0                                & 45.43                                 & 5.47                                \\
            \rowcolor{blue!6}
            \, + MT-GRPO                     & 44.00                                                & \underline{57.0}                                & 45.0                                   & 36.0                                & 38.0                                & \underline{49.61}                     & 7.03                                \\
            \rowcolor{blue!6}
            \, + EnvTuning                   & \underline{46.00}                                    & 54.0                                            & 45.0                                   & \underline{40.0}                    & \underline{45.0}                    & 49.02                                 & \underline{7.81}                    \\
            \rowcolor{orange!6}
            \, + OPD                         & 44.75                                                & 50.0                                            & \textbf{52.0}                          & 39.0                                & 38.0                                & 45.25                                 & 3.91                                \\
            \rowcolor{orange!6}
            \, + OPSD                        & 41.75                                                & 48.0                                            & 43.0                                   & 39.0                                & 37.0                                & 43.95                                 & 3.91                                \\
            \rowcolor{cyan!8}
            \, + \textbf{\ours}              & \textbf{50.00} \improvement{(+16.62)}                & \textbf{60.0} \improvement{(+18.5)}             & \underline{51.0} \improvement{(+12.5)} & \textbf{42.0} \improvement{(+15.0)} & \textbf{47.0} \improvement{(+20.5)} & \textbf{52.34} \improvement{(+8.59)}  & \textbf{9.38} \improvement{(+4.69)} \\
            \bottomrule
        \end{tabular}
    }
    \vspace{2pt}
\end{table*}

\subsection{Main Results}

\Cref{tab:bfcl_v3_results} reports BFCL V3 and WildToolBench performance for both model scales across all baselines. \ours{} obtains higher point estimates than all baselines on average and achieves the best or tied-best result on nearly all individual splits across both benchmarks and model scales.

\paragraph{RL outperforms distillation}
Two patterns stand out in \cref{tab:bfcl_v3_results}. First, RL-based methods (GRPO, MT-GRPO, EnvTuning) systematically outperform distillation-based methods (OPD, OPSD) at both model scales. The clearest signal is OPSD: on Qwen3-4B BFCL Avg it scores $38.75\%$, falling \emph{below} GRPO ($43.63\%$) despite using a privileged ground-truth-conditioned teacher, illustrating an empirical limitation of the teacher-defined objective in this setting.
Second, within the RL family, finer-grained credit assignment helps: on Qwen3-4B BFCL Avg, MT-GRPO ($49.25\%$) and EnvTuning ($47.25\%$) both improve clearly over trajectory-level GRPO ($43.63\%$), yet every token within a step still shares a single advantage. The remaining headroom for dense per-token supervision, on top of verifier-determined directions, is exactly the room \ours is designed to fill.

\paragraph{\ours closes both gaps simultaneously.}
\ours reaches $52.00\%$ on Qwen3-4B-Instruct and $50.00\%$ on Qwen3-8B, surpassing the strongest RL baseline at each scale and achieving the best or tied-best result in 13 of 14 cells. OPD is best on Qwen3-8B \textit{Missing Functions}, while CrEST ties OPD on Qwen3-4B \textit{Missing Parameters}.
The margins are largest on the splits that stress hierarchical credit assignment: on BFCL \emph{Long Context} (the longest-trajectory split), \ours{} improves over the strongest baseline by $+7.0$ points at 4B and $+2.0$ points at 8B; on WildToolBench Session Accuracy (the strictest end-to-end multi-turn metric), by $+0.78$ points at 4B and $+1.57$ points at 8B.
These are exactly the regimes where credit dilution is worst, validating that turn-segmented advantages and entropy-gated modulation address complementary failure modes.

\subsection{Training Dynamics Analysis}

\begin{figure*}[!t]
    \centering
    \begin{subfigure}[t]{0.48\textwidth}
        \centering
        \includegraphics[width=\linewidth]{figures/grpo_vs_teacher.png}
        \caption{Training rollout accuracy.}
        \label{fig:training_curve}
    \end{subfigure}
    \hfill
    \begin{subfigure}[t]{0.48\textwidth}
        \centering
        \includegraphics[width=\linewidth]{figures/gradient_concentration.png}
        \caption{Top-$p$\% token gradient share.}
        \label{fig:gradient_concentration}
    \end{subfigure}
    \caption{Training dynamics on BFCL V3 (Qwen3-4B). (a)~\ours{} converges faster and reaches higher final rollout accuracy than GRPO, exceeding the observed OPSD plateau (dashed) by step 20 while OPSD subsequently declines. (b)~Top-$p$\% token gradient share: OPSD concentrates $>$77\% of gradient on the top-5\% tokens, so its signal is dominated by a few tokens; \ours{}'s entropy gate selectively amplifies high-uncertainty content tokens while limiting amplification on low-uncertainty tokens.}
    \label{fig:training_dynamics}
\end{figure*}

\Cref{fig:training_curve} compares training rollout accuracy on BFCL V3 (Qwen3-4B). \ours reaches $0.60$ accuracy within ${\sim}20$ steps and converges to ${\sim}0.70$, while GRPO rises slowly to ${\sim}0.57$ by step 160. OPSD plateaus at ${\sim}0.49$ and declines thereafter, consistent with the observed limitations of its teacher-defined optimization target discussed by \citet{lu2025onpolicy} and \citet{brown2026sftrlopd}. \ours exceeds this observed plateau by step 20 and continues improving, showing that the direction gate can combine teacher-informed magnitude modulation with verifier-derived token-wise signs in this setting.

\Cref{fig:gradient_concentration} provides a mechanistic explanation for the performance gap. We measure the fraction of total gradient magnitude captured by the top-$p$\% of tokens (ranked by per-token advantage magnitude) for each method. OPSD exhibits extreme concentration: the top-1\% of tokens account for ${\sim}42\%$ of the total gradient signal, and the top-5\% account for ${\sim}77\%$. This matches the gradient-geometric analysis of \citet{brown2026sftrlopd}: a handful of pivot tokens dominate the update, causing instability. GRPO, by contrast, is far more diffuse (top-10\% $\approx$ 31\%), but at the cost of treating all tokens equally. \ours{} occupies the desirable middle ground: more concentrated than GRPO (top-10\% $\approx$ 57\%), confirming that the self-teacher signal provides meaningful token differentiation, but far less extreme than OPSD, confirming that the entropy gate prevents concentration collapse.

\subsection{Ablation Studies}

\begin{wraptable}{r}{0.5\columnwidth}
    \vspace{-0.8em}
    \centering
    \small
    \caption{Hierarchical ablation on BFCL V3 (Qwen3-4B-Instruct). Neither inter-turn nor intra-turn credit assignment alone recovers the full \ours{} performance.}
    \label{tab:hierarchy_ablation_bfcl}
    \resizebox{\linewidth}{!}{
        \begin{tabular}{lccccc}
            \toprule
            \textbf{Method}            & \textit{Avg}   & \textit{Base} & \textit{M.Func} & \textit{M.Param} & \textit{L.Ctxt} \\
            \midrule
            GRPO (baseline)            & 43.63          & 53.5          & 42.5            & 31.5             & 47.0            \\
            \, + Inter-turn only       & 47.88          & 62.0          & 44.0            & 35.5             & 50.0            \\
            \, + Intra-turn only       & 48.75          & 61.0          & 50.0            & 32.0             & 52.0            \\
            \, + Both (\textbf{\ours}) & \textbf{52.00} & \textbf{67.0} & \textbf{48.0}   & \textbf{38.0}    & \textbf{60.0}   \\
            \bottomrule
        \end{tabular}
    }
    \vspace{-1.2em}
\end{wraptable}

\paragraph{Effect of hierarchical credit decomposition.}

\Cref{tab:hierarchy_ablation_bfcl} isolates the contribution of each credit assignment level.
The \textit{Intra-turn only} variant retains the trajectory-level GRPO group-relative advantage $A_i^{\mathrm{traj}}$ for every token and applies only the teacher modulation, $\mathcal{A}_t=A_i^{\mathrm{traj}}\phi_t$; it uses neither turn segmentation nor process rewards.
Inter-turn segmentation alone improves average accuracy over GRPO ($43.63 \to 47.88$), with the largest gain on \textit{Base} ($+8.5$) where cross-turn reward dilution is most severe.
Intra-turn modulation alone yields $+5.12$ points, with particular strength on \textit{Miss Func} ($+7.5$) where distinguishing content tokens from format tokens matters most. Combining both levels achieves $+8.37$ points ($43.63 \to 52.00$), with \textit{Long Context} showing the strongest synergy.
Neither component alone approaches the full model, confirming that the two levels address complementary failure modes: turn segmentation provides correct per-turn reward direction, while entropy gating selectively amplifies high-value tokens within each turn.

% Allow the wrapped table to move when the remaining page space is too short.
\begin{wraptable}{R}{0.5\columnwidth}
    \vspace{-0.8em}
    \centering
    \small
    \caption{Gating ablation on BFCL V3 (Qwen3-4B-Instruct). Removing the direction gate weakens verifier anchoring; removing the entropy gate causes gradient concentration on format tokens; removing both leads to training instability.}
    \label{tab:gating_ablation}
    \resizebox{\linewidth}{!}{
        \begin{tabular}{@{}lccccc@{}}
            \toprule
            \textbf{Variant}         & \textit{Avg}   & \textit{Base} & \textit{M.Func} & \textit{M.Param} & \textit{L.Ctxt} \\
            \midrule
            \ours                    & \textbf{52.00} & \textbf{67.0} & \textbf{48.0}   & \textbf{38.0}    & \textbf{60.0}   \\
            \, +  w/o Direction gate & 46.75          & 60.0          & 47.0            & 28.0             & 52.0            \\
            \, +  w/o Entropy gate   & 46.25          & 59.0          & 50.0            & 27.0             & 49.0            \\
            \, +  w/o Both gates     & 43.50          & 53.0          & 43.0            & 27.0             & 51.0            \\
            \bottomrule
        \end{tabular}
    }
    \vspace{-1.2em}
\end{wraptable}

\paragraph{Both gates are necessary.}
\Cref{tab:gating_ablation} ablates the direction gate and entropy gate individually and jointly. Removing the direction gate drops average accuracy from $52.00\%$ to $46.75\%$ ($-5.25$), with \textit{Miss Param} suffering the largest decline ($38.0 \to 28.0$), consistent with the teacher overriding verifier signals on ambiguous parameter tokens. Removing the entropy gate yields a similar degradation ($52.00 \to 46.25$), with \textit{Long Context} dropping from $60.0$ to $49.0$, where long sequences amplify gradient concentration on format tokens. Removing both gates ($43.50\%$) falls to GRPO-level performance, confirming that uncontrolled self-distillation provides no benefit over standard RL. A sensitivity analysis of $\lambda$ is provided in \cref{app:lambda_sensitivity}.

\section{Conclusion}

We presented \ours, a hierarchical credit assignment framework for multi-turn multi-step agent training. By decomposing per-token advantages into turn-segmented verified rewards and entropy-gated self-teacher modulation, \ours resolves the two-level credit assignment problem that afflicts standard RL in multi-turn settings: inter-turn reward dilution and intra-turn token uniformity. Experiments on BFCL V3 and WildToolBench show that \ours achieves the strongest aggregate BFCL results and higher WildToolBench scores in the evaluated runs, while ablations indicate that the two levels address complementary failure modes. Our central finding is that a self-teacher need not determine token-level reinforcement or suppression to be useful---restricting it to magnitude modulation, gated by student uncertainty, provides dense credit while keeping each sampled token's scalar advantage sign anchored to verified rewards.

Several directions for future work are worth exploring. The magnitude-only modulation principle demonstrated here could be applied to other structured generation tasks beyond tool use, such as multi-hop retrieval or collaborative dialogue, where hierarchical outcome structure is present. Additionally, the design space between our fixed self-teacher and fully online co-training remains largely unexplored; adaptive scheduling of teacher influence strength $\lambda$ across training stages may yield further gains. Finally, combining \ours with complementary advances in reward design (e.g., progress reward from~\citet{lu2026dont}) could compound the benefits of both finer reward signals and better credit allocation.

\section{Limitations}

Our evaluation is conducted on two multi-turn tool-use benchmarks at the 4B and 8B scale. While these represent distinct interaction paradigms (structured function-calling and naturalistic agentic sessions), broader validation across additional benchmarks and larger model scales would further strengthen the generality claims. Additionally, the entropy gate uses per-token surprisal as a proxy for token importance, which is effective for separating format tokens from content tokens in tool-use trajectories but remains a heuristic that may not optimally capture token informativeness in all generation contexts.

\section{Ethics Statement}

Our work improves the training efficiency of LLM-based tool-calling agents, which carries dual-use risks common to agent capability research. More capable agents could be misused for automated social engineering, unauthorized system access, or large-scale data extraction via tool APIs. We note that \ours{} does not introduce new capabilities beyond what existing models can already attempt---it improves the \emph{reliability} of tool-use within sanctioned benchmarks. Nevertheless, we encourage practitioners to pair improved agent training with robust safety measures, including action sandboxing, human-in-the-loop approval for sensitive tool calls, and rate limiting on API access. All experiments in this work are conducted on public benchmarks with synthetic tool environments that do not interact with real-world systems.

\section*{Acknowledgments}
This work was supported in part by the National Natural Science Foundation of China (No. 62602580), the Ningbo Youth Science and Technology Innovation Leading Talent Program (No. 2025QL059), the Ant Group Research Fund, and the ``Pioneer and Leading Goose'' R\&D Program of Zhejiang (No. 2025C02037).

We acknowledge the use of AI assistants (Claude, Anthropic) for writing polish and \LaTeX{} formatting during the preparation of this manuscript. All scientific content, experimental design, and technical contributions are solely the work of the authors.

%%%%%%%%%%%%%%%%%%%%%%%%%%%%%%%%%%%%%%%%%%%%%%%%%%%%%%%%%%%%%%%%%%
%%%%%%%%%%%%%%%%%%%%%%%%%%%%%%%%%%%%%%%%%%%%%%%%%%%%%%%%%%%%%%%%%%
%%%%%%%%%%%%%%%%%%%%%%%%%%%%%%%%%%%%%%%%%%%%%%%%%%%%%%%%%%%%%%%%% 

% reference
\newpage
\bibliographystyle{config/antgroup}
\bibliography{doc/reference}

% appendix
\newpage
\appendix
\section{Gradient Analysis}
\label{app:gradient_analysis}

We prove the three properties stated in \S\ref{sec:method}. Recall the structured per-token advantage:
\begin{equation}
    \begin{aligned}
        \mathcal{A}_t & = A_{[t]}^{\mathrm{turn}} \cdot \phi_t,        \\
        \phi_t        & = 1 + \lambda_t^{\mathrm{eff}} \cdot (w_t - 1)
    \end{aligned}
\end{equation}
and the resulting policy gradient:
\begin{equation}
    \nabla_\theta J(\theta) = \mathbb{E}\!\left[\sum_t \mathcal{A}_t \nabla_\theta \log \pi_\theta(y_t)\right]
    \label{eq:app_policy_gradient}
\end{equation}

\subsection{(P1) Direction Preservation}

\begin{proposition}[Direction Preservation]
    For all tokens $t$, $\mathrm{sign}(\mathcal{A}_t) = \mathrm{sign}(A_{[t]}^{\mathrm{turn}})$.
\end{proposition}

\begin{proof}
    We show $\phi_t > 0$ always holds, so multiplication by $\phi_t$ preserves the sign of $A_{[t]}^{\mathrm{turn}}$.

    \textbf{Case 1:} $g_t^{\mathrm{dir}} = 0$ (direction gate off). Then $\lambda_t^{\mathrm{eff}} = 0$ by Eq.~\ref{eq:lambda_eff}, so $\phi_t = 1 > 0$.

    \textbf{Case 2:} $g_t^{\mathrm{dir}} = 1$ (direction gate on). This requires $\mathrm{sign}(A_{[t]}^{\mathrm{turn}}) \cdot \Delta_t > 0$ by Eq.~\ref{eq:direction_gate}, which implies $\exp(\mathrm{sign}(A_{[t]}^{\mathrm{turn}}) \cdot \Delta_t) > 1$. After clipping, $w_t \in [1, 1+\epsilon]$, so $w_t - 1 \geq 0$. Since $\lambda_t^{\mathrm{eff}} \geq 0$:
    \begin{equation}
        \phi_t = 1 + \underbrace{\lambda_t^{\mathrm{eff}}}_{\geq\, 0} \underbrace{(w_t - 1)}_{\geq\, 0} \geq 1 > 0
    \end{equation}

    In both cases $\phi_t > 0$, so $\mathrm{sign}(\mathcal{A}_t) = \mathrm{sign}(A_{[t]}^{\mathrm{turn}} \cdot \phi_t) = \mathrm{sign}(A_{[t]}^{\mathrm{turn}})$.
\end{proof}

\textbf{Interpretation.} The scalar coefficient of each token's policy-gradient contribution has the same sign as the inter-turn-only coefficient (i.e., \ours{} with $\phi_t{=}1$). The teacher signal modulates the \emph{magnitude} of each token's contribution but cannot flip its scalar sign, yielding a local, token-wise anchoring property.

\subsection{(P2) Absolute Perturbation Bound}

\begin{proposition}[Absolute Perturbation Bound]
    For a sampled trajectory, let
    $g_t = A_{[t]}^{\mathrm{turn}}\nabla_\theta\log\pi_\theta(y_t\mid h_t)$ denote the per-token inter-turn gradient contribution. Let $\widehat{\nabla_\theta J}$ and $\widehat{\nabla_\theta J}_{\mathrm{TG}}$ denote the reweighted and inter-turn-only single-trajectory gradients, respectively. The teacher-induced perturbation satisfies:
    \begin{equation}
        \left\|\sum_t(\phi_t-1)g_t\right\|
        \leq \lambda\epsilon\sum_t\|g_t\|.
    \end{equation}
    The bound is absolute; it is not relative to $\left\|\sum_t g_t\right\|$.
\end{proposition}

\begin{proof}
    The difference between the reweighted and inter-turn-only gradients on one sampled trajectory is:
    \begin{equation}
        \widehat{\nabla_\theta J} - \widehat{\nabla_\theta J}_{\mathrm{TG}}
        = \sum_t(\phi_t-1)g_t.
    \end{equation}

    If $g_t^{\mathrm{dir}}=0$, then $\phi_t=1$. If $g_t^{\mathrm{dir}}=1$, then $\phi_t-1=\lambda_t^{\mathrm{eff}}(w_t-1)$, with $\lambda_t^{\mathrm{eff}}\leq\lambda$ and $w_t-1\leq\epsilon$ by Eqs.~\ref{eq:lambda_eff} and \ref{eq:weight}. Thus, for every token, $|\phi_t-1|\leq\lambda\epsilon$. Applying the triangle inequality gives:
    \begin{equation}
        \left\|\sum_t(\phi_t-1)g_t\right\|
        \leq \sum_t|\phi_t-1|\,\|g_t\|
        \leq \lambda\epsilon\sum_t\|g_t\|.
    \end{equation}
\end{proof}

\textbf{Interpretation.} The teacher-induced perturbation is controlled in absolute magnitude by the sum of per-token inter-turn gradient norms. With the default settings ($\lambda=0.3$, $\epsilon=0.28$), each token's advantage coefficient can be amplified by at most $8.4\%$; this is not a relative bound on the aggregate gradient, which may be small because token contributions cancel. Reducing $\lambda$ still makes the per-token modulation arbitrarily weak, smoothly approaching the inter-turn-only ablation.

\subsection{(P3) Sign-Consistent Amplification}

\begin{proposition}[Sign-Consistent Amplification]
    When $g_t^{\mathrm{dir}} = 1$, $\phi_t \geq 1$. That is, the teacher can only \emph{amplify} the gradient magnitude along the verifier-approved direction, never suppress it.
\end{proposition}

\begin{proof}
    This was established in the proof of (P1), Case~2: $g_t^{\mathrm{dir}} = 1 \implies w_t \geq 1 \implies \phi_t = 1 + \lambda_t^{\mathrm{eff}}(w_t - 1) \geq 1$.
\end{proof}

\textbf{Interpretation.} Combined with (P1), this means: for every token, either the teacher has no effect ($\phi_t = 1$, when the direction gate is off) or the teacher \emph{increases} the magnitude of the update in the direction already approved by the verified reward ($\phi_t > 1$). The teacher signal is therefore a selective token-level amplifier.

\begin{table*}[!t]
    \centering
    \small
    \caption{Comparison with hybrid RL and self-distillation methods. ``Token direction'' denotes whether a sampled token is reinforced or suppressed by its scalar advantage, not the direction of the aggregate parameter gradient.}
    \resizebox{\textwidth}{!}{
        \begin{tabular}{lcccc}
            \toprule
            \textbf{Method} & \textbf{Objective signal} & \textbf{Token direction} & \textbf{Magnitude} & \textbf{Supervision} \\
            \midrule
            MT-GRPO~\citep{wei2025reinforcing} & $A_t=A_k^{\mathrm{cum}}$ & turn verifier & uniform per turn & intermediate + outcome rewards \\
            EnvTuning~\citep{lu2026dont} & $A_t=A_{\mathrm{traj}}^V$ & process verifier & uniform per trajectory & process rewards \\
            RLSD~\citep{yang2026rlsd} & $A_t=A_{\mathrm{traj}}^V m_t$ & trajectory verifier & self-teacher & terminal reward + reference \\
            SDAR~\citep{lu2026self} & $\mathcal{L}_{\mathrm{RL}}+\lambda\mathcal{L}_{\mathrm{SD}}$ & verifier + teacher branch & gated distillation & reward + retrieved skills \\
            OPD~\citep{lu2025onpolicy} & $A_t=\Delta_t^{\mathrm{ext}}$ & external teacher & external teacher & same-family 235B teacher \\
            OPSD~\citep{zhao2026self} & $A_t=\Delta_t^{\mathrm{self}}$ & self-teacher & self-teacher & privileged context \\
            \textbf{\ours{} (Ours)} & $A_t=A_k^V\phi_t$ & \textbf{independent turn verifier} & \textbf{gated self-teacher} & per-turn reward + ground truth \\
            \bottomrule
        \end{tabular}
    }
    \label{tab:hybrid_comparison}
\end{table*}

\subsection{Discussion: Relationship to Adaptive Preconditioning}

The modulation factor $\phi_t$ can be interpreted as a form of token-level adaptive preconditioning on the policy gradient. Standard GRPO applies a uniform scale to all tokens within a trajectory (or turn). Our method replaces this with a per-token scale that is:
\begin{itemize}
    \item Larger on tokens where the privileged teacher agrees with the verifier direction \emph{and} the student is uncertain (high surprisal).
    \item Equal to $1$ on tokens where the teacher disagrees, or the student is already confident.
\end{itemize}

This is analogous to how natural gradient methods or Adam apply per-parameter adaptive learning rates. In our case, the teacher--student divergence gated by student uncertainty serves as the ``curvature proxy,'' yielding adaptive scalar reweighting at the token level rather than the parameter level.

The key difference from OPSD is that OPSD uses the teacher signal as the advantage itself (both direction and magnitude), whereas our method uses it only as a magnitude modulator on top of verified reward. Thus, OPSD's token-level optimization target is teacher-defined, while CrEST keeps each token's scalar update sign anchored to the verifier.

\paragraph{Relationship to hybrid self-distillation.}
As summarized in \cref{tab:hybrid_comparison}, RLSD is the closest predecessor: it combines verifier-derived signs with self-teacher magnitudes, but operates at the sequence level without uncertainty gating and is evaluated only in single-step settings. \ours{} adds independent per-turn credit and uncertainty gating. SDAR targets multi-turn agents with gated token supervision, but adds a separate distillation loss to trajectory-level RL, whereas \ours{} uses the single multiplicative advantage $A_k^V\phi_t$ without an auxiliary teacher objective.

\section{Complete Algorithm}
\label{app:algorithm}

\Cref{alg:crest} provides the complete pseudocode for a single \ours training step, consolidating all equations from \cref{sec:method} into an executable procedure.

\begin{algorithm*}[ht]
    \caption{\ours Training Step}
    \label{alg:crest}
    \begin{algorithmic}[1]
        \REQUIRE Prompt batch $\mathcal{D}$, policy $\pi_\theta$, group size $G$, hyperparams $\lambda, \tau, \epsilon, \rho$
        \FOR{each prompt $x \in \mathcal{D}$}
        \STATE Sample $G$ rollouts $\{y^{(i)}\}_{i=1}^G \sim \pi_\theta(\cdot|x)$
        \STATE Obtain per-turn verified rewards $\{R_k^{(i)}\}$
        \STATE \textbf{// Inter-turn: turn-segmented advantage}
        \FOR{each turn $k$}
        \STATE $A_k^{(i)} \gets (R_k^{(i)} - \mu_k) / (\sigma_k + \epsilon)$ \hfill \textit{(\cref{eq:turn_advantage})}
        \ENDFOR
        \STATE \textbf{// Intra-turn: entropy-gated self-teacher}
        \STATE Compute privileged teacher logprobs $\log \pi_T(y_t \mid h_t^T)$
        \FOR{each token $t$ in rollout $i$}
        \STATE $\Delta_t \gets (\log \pi_T - \log \pi_\theta) / \tau$ \hfill \textit{(\cref{eq:delta})}
        \STATE $w_t \gets \mathrm{clip}(\exp(\mathrm{sign}(A_{[t]}^{\mathrm{turn}}) \cdot \Delta_t),\; 1{-}\epsilon,\; 1{+}\epsilon)$ \hfill \textit{(\cref{eq:weight})}
        \STATE $g_t^{\mathrm{dir}} \gets \mathbf{1}[\mathrm{sign}(A_{[t]}^{\mathrm{turn}}) \cdot \Delta_t > 0]$ \hfill \textit{(\cref{eq:direction_gate})}
        \STATE $g_t^{\mathrm{ent}} \gets \sigma((u_t - \mathbb{E}[u]) / (\mathrm{std}(u) + \epsilon))$ \hfill \textit{(\cref{eq:entropy_gate})}
        \STATE $\lambda_t^{\mathrm{eff}} \gets \mathrm{clip}(\lambda \cdot g_t^{\mathrm{dir}} \cdot (1 + \rho(2g_t^{\mathrm{ent}} - 1)),\; 0,\; \lambda)$ \hfill \textit{(\cref{eq:lambda_eff})}
        \STATE $\mathcal{A}_t \gets A_{[t]}^{\mathrm{turn}} \cdot (1 + \lambda_t^{\mathrm{eff}} \cdot (w_t - 1))$
        \ENDFOR
        \ENDFOR
        \STATE Update $\theta$ via policy gradient with per-token advantages $\{\mathcal{A}_t\}$
    \end{algorithmic}
\end{algorithm*}

\section{Experimental Details}
\label{app:implementation}

\subsection{Training Configuration}

All experiments use the following shared configuration:

\begin{itemize}[nosep, leftmargin=12pt]
    \item \textbf{Hardware}: Single-node $8{\times}$ H200 GPUs. For methods requiring a teacher (OPSD, \ours{}), 4 GPUs serve the teacher via SGLang and 4 GPUs handle student training and rollout generation, with tensor parallelism $= 2$ and sequence parallelism enabled. For OPD with Qwen3-4B-Instruct, the external teacher (Qwen3-235B-A22B-Instruct) requires an additional 8 GPUs for deployment, resulting in a 16-GPU setup (8 for teacher, 8 for student).
    \item \textbf{Rollout}: On-policy sampling with group size $G{=}16$, temperature $1.0$, max generation length $10{,}000$ tokens per rollout.
    \item \textbf{Batch size}: $32$ prompts per step $\times$ $16$ samples per prompt $= 512$ rollouts per step.
    \item \textbf{Optimizer}: Adam with $\beta_1{=}0.9$, $\beta_2{=}0.999$, weight decay $0.01$, gradient clipping $1.0$.
    \item \textbf{Learning rate}: $1{\times}10^{-6}$, constant schedule.
    \item \textbf{Training steps}: The main Qwen3-4B BFCL V3 runs use method-specific schedules: $160$ steps for GRPO, $150$ for MT-GRPO, $160$ for EnvTuning, $100$ for OPD, $89$ for OPSD, and $160$ for \ours. We do not assume a single universal step count across methods.
    \item \textbf{PPO clipping}: $\epsilon_{\mathrm{low}}{=}0.2$, $\epsilon_{\mathrm{high}}{=}0.28$.
    \item \textbf{KL penalty}: $\beta_{\mathrm{KL}}{=}0.0$ (no explicit KL regularization; clipping provides implicit constraint).
\end{itemize}

\begin{table*}[t]
    \centering
    \small
    \caption{Training cost on BFCL V3 with Qwen3-4B-Instruct. All runs use H200 GPUs. GRPO wall-clock and GPU-hours are normalized to the reported $29.1$ minutes per step over the unified $160$-step schedule.}
    \label{tab:training_cost}
    \resizebox{\textwidth}{!}{
        \begin{tabular}{lclcrrrrr}
            \toprule
            \textbf{Method} & \textbf{Peak GPUs} & \textbf{Allocation} & \textbf{Extra teacher} & \textbf{Steps} & \textbf{Wall-clock (h)} & \textbf{min/step} & \textbf{GPU-h} & \textbf{BFCL Avg.} \\
            \midrule
            GRPO       & 8  & 8 student             & no            & 160 & 77.6  & 29.1 & 621 & 43.63 \\
            MT-GRPO    & 8  & 8 student             & no            & 150 & 81.2  & 32.5 & 650 & 49.25 \\
            EnvTuning  & 8  & 8 student             & no            & 160 & 88.3  & 33.1 & 707 & 47.25 \\
            OPD        & 16 & 8 student + 8 teacher & 235B external & 100 & 22.4  & 13.5 & 359 & 44.50 \\
            OPSD       & 8  & 4 student + 4 teacher & self          & 89  & 31.0  & 20.9 & 248 & 38.75 \\
            \ours{}    & 8  & 4 student + 4 teacher & self          & 160 & 105.0 & 39.4 & 840 & 52.00 \\
            \bottomrule
        \end{tabular}
    }
\end{table*}

\paragraph{Accuracy--compute trade-off.}
\Cref{tab:training_cost} shows that \ours{} has the same eight-GPU peak footprint as GRPO, split between four student GPUs and four self-teacher GPUs. OPD instead requires a second node and a same-family 235B external teacher, so its lower total GPU-hours are not directly comparable to the eight-GPU methods. Relative to GRPO, \ours{} increases average step time from $29.1$ to $39.4$ minutes ($+35.4\%$), total training time from $77.6$ to $105.0$ hours ($+35.3\%$), and compute from $621$ to $840$ GPU-hours ($+35.3\%$). Teacher scoring is pipelined with student rollout, so the overhead is not a fully serialized teacher forward pass; it also reflects the reduction from eight to four student-side GPUs. This additional compute yields a $+4.12$ BFCL improvement over the inter-turn-only variant ($47.88$ to $52.00$) and $+8.37$ over GRPO ($43.63$ to $52.00$), representing an explicit accuracy--compute trade-off.

For \ours{}-specific hyperparameters: $\lambda{=}0.3$ (modulation strength), $\epsilon{=}0.28$ (clip bound), $\tau{=}2.0$ (temperature for $\Delta_t$), and $\rho{=}0.5$ (entropy gate range). Only $\lambda$ is tuned; the remaining values are fixed design choices across all benchmarks and model scales.

The self-teacher is constructed by conditioning the same model on ground-truth tool-call results for preceding turns, providing privileged context without requiring an external teacher model or tokenizer alignment.

\subsection{OPSD Stabilization and Teacher Update Strategies}
\label{app:opsd_stabilization}

\begin{figure*}[!t]
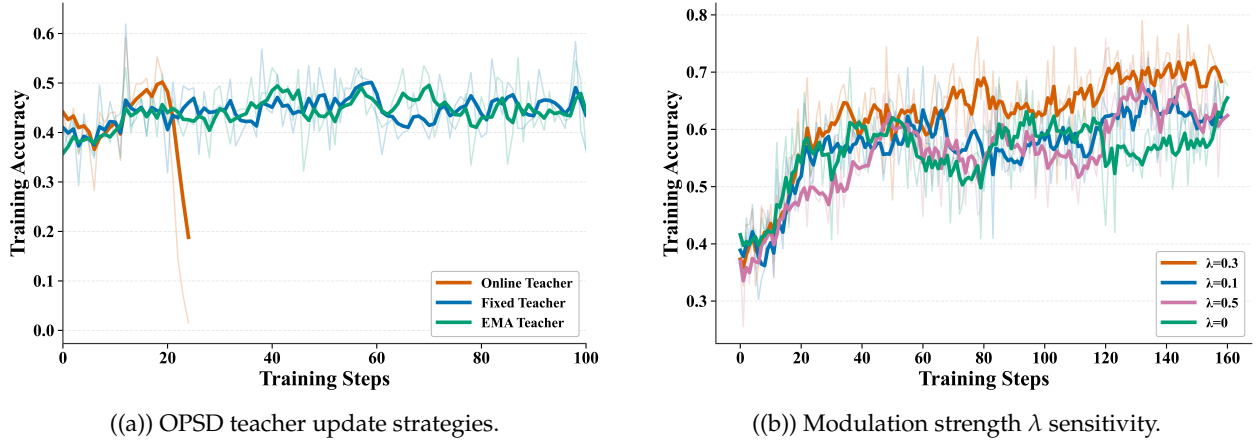

    \centering
    \begin{subfigure}[t]{0.48\textwidth}
        \centering
        \includegraphics[width=\linewidth]{figures/three_teachers.png}
        \caption{OPSD teacher update strategies.}
        \label{fig:three_teachers}
    \end{subfigure}
    \hfill
    \begin{subfigure}[t]{0.48\textwidth}
        \centering
        \includegraphics[width=\linewidth]{figures/ablation_lambda.png}
        \caption{Modulation strength $\lambda$ sensitivity.}
        \label{fig:ablation_lambda}
    \end{subfigure}
    \caption{Analysis of OPSD teacher update strategies and modulation strength on BFCL V3 (Qwen3-4B). (a)~Comparison of OPSD teacher update strategies: the online teacher collapses catastrophically within ${\sim}20$ steps, while fixed and EMA teachers are both stable, with the fixed teacher achieving marginally higher and more consistent accuracy; we adopt the fixed teacher with per-token KL clipping (threshold $0.2$) for the OPSD baseline in all experiments. (b)~Effect of the modulation strength $\lambda$ on training accuracy: $\lambda{=}0.3$ provides the best trade-off; higher values over-concentrate gradients.}
\end{figure*}

Directly applying OPSD~\citep{zhao2026self} to multi-turn tool-use trajectories results in training collapse within ${\sim}20$ steps, consistent with findings reported by the original authors. To establish a fair OPSD baseline, we investigated three teacher update strategies (\cref{fig:three_teachers}):

\begin{itemize}[nosep, leftmargin=12pt]
    \item \textbf{Online teacher} (updated with policy): Catastrophic collapse by step 20---training accuracy drops from ${\sim}0.45$ to near zero, caused by pivot-token concentration shifting direction each update.
    \item \textbf{Fixed teacher} (frozen at initialization): The most stable strategy, maintaining training accuracy around $0.44$--$0.48$ throughout. However, the static teacher distribution defines the token-level optimization target and produces an observed plateau.
    \item \textbf{EMA teacher} (exponential moving average, $\alpha{=}0.995$): Comparable stability to the fixed teacher but no improvement---the teacher converges to a lagged copy of the student, erasing the privileged information advantage.
\end{itemize}

Based on these results, our OPSD baseline adopts the fixed teacher with per-token KL clipping (clip threshold $0.2$) as recommended by \citet{zhao2026self}. This prevents collapse but leaves the student at the observed plateau associated with a teacher-defined optimization target shown in \cref{fig:three_teachers}. This empirical observation directly motivates \ours{}'s design: rather than using the teacher signal as the optimization target, we use it solely as a gradient magnitude modulator anchored to the verified token sign.

\subsection{Modulation Strength Sensitivity}
\label{app:lambda_sensitivity}

\Cref{fig:ablation_lambda} varies the modulation strength $\lambda$, which controls how aggressively the entropy gate selectively amplifies token updates. At $\lambda{=}0$ (inter-turn-only \ours{}, $\phi_t{=}1$), final accuracy is ${\sim}0.63$. At $\lambda{=}0.3$, entropy-gated modulation raises final accuracy to ${\sim}0.69$ (+6\%) by emphasizing high-uncertainty content tokens and reducing amplification on low-entropy format tokens. At $\lambda{=}0.5$, over-concentration causes early gains but stagnation below $\lambda{=}0.1$ by step 80, ending at ${\sim}0.61$. The sweet spot at $\lambda{=}0.3$ confirms that \emph{calibrated} selective amplification, not maximal concentration, best resolves intra-turn credit ambiguity.

\end{document}